\documentclass[10pt,twocolumn,letterpaper]{article}

\usepackage[margin=0.75in,columnsep=0.25in]{geometry}
\usepackage{amsmath}
\usepackage{amssymb}
\usepackage{amsthm}
\usepackage{graphicx}
\usepackage{array}
\usepackage{url}
\usepackage{times}
\usepackage{booktabs}
\usepackage[hidelinks]{hyperref}
\usepackage{newunicodechar}
\newunicodechar{−}{-}
\newunicodechar{⁻}{\textsuperscript{-}}
\newcommand{\SE}{\mathrm{SE}(3)}

\newcommand{\svec}{\mathbf{s}}
\newcommand{\gvec}{\mathbf{g}}

\makeatletter
\renewcommand{\@seccntformat}[1]{\csname the#1\endcsname.\hspace{0.5em}}
\makeatother

\newtheorem{proposition}{Proposition}

\newcommand{\wfA}{lorem ipsum dolor sit amet consectetur adipiscing elit sed do}
\newcommand{\wfB}{eiusmod tempor incididunt ut labore et dolore magna aliqua enim}
\newcommand{\wfC}{ad minim veniam quis nostrud exercitation ullamco laboris nisi ut}
\newcommand{\wfD}{aliquip ex ea commodo consequat duis aute irure in reprehenderit}
\newcommand{\wfE}{voluptate velit esse cillum dolore eu fugiat nulla pariatur ex}
\newcommand{\wfF}{sint occaecat cupidatat non proident sunt in culpa qui officia}
\newcounter{wfn}\newcounter{wfi}\newcounter{wfp}
\newcommand{\wordfill}[1]{%
  \setcounter{wfn}{0}\setcounter{wfi}{1}\setcounter{wfp}{0}%
  \loop\ifnum\value{wfn}<#1\relax
    \ifcase\value{wfi}\or\wfA\or\wfB\or\wfC\or\wfD\or\wfE\or\wfF\fi\ %
    \addtocounter{wfn}{10}\stepcounter{wfi}\stepcounter{wfp}%
    \ifnum\value{wfi}>6\setcounter{wfi}{1}\fi
    \ifnum\value{wfp}>7\par\setcounter{wfp}{0}\fi
  \repeat\par}

\title{\vspace{-1.2cm}What Symmetry Buys a Learned Motion Planner}
\author{Andrea Emir Sevincel\\
\small Shanghai Jiao Tong University}
\date{}

\begin{document}
\maketitle

\begin{abstract}
\noindent
Learning-based motion planners pay at training what classical planners
pay per query. Trained in world coordinates, they relearn the same
motion at every position and orientation. Existing work restores the
missing rigid-body equivariance in the training data, in the inference
operator, or in the weights, and each carries a cost. We ask how much of
that equivariance the planning query supplies for free.

A start $\mathbf{s}$ and a goal $\mathbf{g}$ determine a frame in closed
form, with origin at their midpoint and first axis along
$\mathbf{g}-\mathbf{s}$. Expressing trajectory and obstacles in that
frame removes three translations and two rotations of $\mathrm{SE}(3)$,
at initialisation, for one cross product per query and with no
constraint on the architecture. A single rotation about the start-goal
axis remains, and no continuous rule removes it. On a cluttered 3D
benchmark, holding architecture, data and budget fixed, the frame raises
the held-out collision-free rate from $14.60\%$ to $51.10\%$, where a
straight segment from start to goal scores $15.6\%$ and the world-frame
model does not beat it. We build all three mechanisms for the residual
rotation and each is worth under a point, though the equivariant
backbone reaches any given level two to three times sooner. What the
representation supplies therefore dominates what any mechanism enforces,
and the standard diagnostic does not see the difference: two models with
indistinguishable non-equivariance residuals differ by $28$ points.
Calibrated against a non-symmetry intervention, the frame is not even the
largest effect available, since local geometry is worth $+40.0$ where the
frame is worth $+36.5$.

\end{abstract}

\section{Introduction}

Motion planning is a fundamental problem in robotics, requiring the generation of
collision-free, kinematically feasible trajectories from a start to a goal
configuration in complex environments. Equivariance to rigid-body transformations
is a cornerstone of generalization in this setting. It ensures that a planner
solves structurally identical tasks identically, regardless of where in the world
they are posed. Sampling-based~\cite{kuffner2000rrtconnect} and
optimization-based~\cite{schulman2014trajopt,zucker2013chomp} algorithms are
equivariant by construction, since they act on the geometry directly, and they
offer completeness or convergence guarantees. Their per-query computational cost
limits real-time applicability. Learning-based
planners~\cite{carvalho2023mpd,fishman2022mpinets,janner2022diffuser,qureshi2019mpnet}
pay at training what those planners pay per query, and have emerged as a
promising alternative. They do not inherit the symmetry. A network trained on
world coordinates relearns the same motion at every position and orientation.

Existing work restores it in one of three parts of the pipeline. In the data,
through augmentation. In the inference operator, through frame averaging or
learned canonicalization~\cite{puny2022fa, kaba2023canon}. Or in the weights,
through architectures equivariant by
construction~\cite{cohen2016gcnn,deng2021vn,fuchs2020se3t,thomas2018tfn}, as
adopted in recent robot learning systems~\cite{ryu2023edf,wang2024equidp}. Each
buys the symmetry at a price: a larger training distribution, additional forward
passes per query, or a restricted hypothesis class. Before paying any of them,
it is worth asking how much of the symmetry the problem statement already
supplies.

A planning query names a start and a goal, and those two points determine a canonical
frame. One can place the origin at their midpoint, align the first axis with the
direction between them, and express every waypoint and every obstacle in that
frame. Three translational and two rotational degrees of freedom vanish exactly,
at initialisation, at the cost of one cross product per query and with no constraint on the
architecture. The six numbers that conditioned the network on the query collapse
to one invariant scalar. The frame needs a second axis to be determined, and ours
takes it from a fixed world axis, so rotating the whole problem rolls the frame
rather than leaving it unchanged. What survives is exactly that rotation: $\SE$
collapses to $\mathrm{SO}(2)$. No continuous rule can fix the
remainder~\cite{milnor1978hairy}, so the sixth degree of freedom is left to one
of the three mechanisms above, and we build all three. Fig.~\ref{fig:teaser} takes the group apart
one stage at a time.

We build all three on a cluttered 3D benchmark, hold the architecture, data and
training budget fixed, and score every model on problems it never saw,
alongside a straight line from start to goal and three classical planners. The
frame is worth $+36.5$ points of held-out collision-free rate; the three
mechanisms for the sixth are worth $+0.8$, $+0.68$ and $+0.15$. That line scores
$15.6\%$, which the world-frame model does not reach.

Our contributions are as follows.

\begin{itemize}\itemsep2pt
  \item \textbf{Five of six degrees of freedom are free.} The start and the goal determine a canonical frame, and expressing the task in it removes five degrees of freedom exactly, at initialisation, for one cross product per query and with no change to the architecture, the data or the budget. It raises the held-out collision-free rate from 14.60\% to 51.10\%, against 15.6\% for a straight segment from start to goal, which the world-frame model does not beat. On an L-shaped rigid body, which has an orientation as well as a position, the reduced model reaches 2.9× the world-frame model at the same budget. Replacing flow matching with DDPM, at the same number of forward passes per sample, moves the reduced model to 50.52\% against 13.51\% for the world frame, so the effect does not belong to the sampler.

  \item \textbf{All three mechanisms for the sixth, measured against each
  other.} Frame averaging and the SO(2)-equivariant backbone were trained to convergence, and neither is worth more than a point. Roll augmentation was measured only at a fixed budget, where its +0.8 already sat inside the spread across three seeds. The SO(2)-equivariant backbone reaches any given level two to three times sooner, so what the symmetry buys is optimisation rather than accuracy.

  \item \textbf{A bound on what post-hoc symmetrisation can remove, and its limits as a predictor.}
  We measure the non-equivariance residual $r$ of the trained field and prove
that no symmetrisation of that field can remove more than
$r^2\lVert \mathcal{A}f \rVert^2$ of squared error, and that frame averaging
reaches this bound. The reduced model has $r = 0.017$. That bound is under
$3 \times 10^{-4}$ of the field's squared magnitude, and the model fails on
$48.9\%$ of its samples. The residual does not then predict what a mechanism is worth: two models with indistinguishable r differ by 28 points.
  
  \item \textbf{Calibration against a non-symmetry intervention.}
  Supplying each waypoint with its distance and direction to the nearest obstacle is worth +40.0 points where the frame is worth +36.5, so the representation change studied here is not the largest effect available on this benchmark. With local geometry supplied the frame still adds +24.8, and the two together reach 79.37, against the 91.1 that independence would predict.

\end{itemize}

\begin{figure}[t]
\centering
\includegraphics[width=\columnwidth]{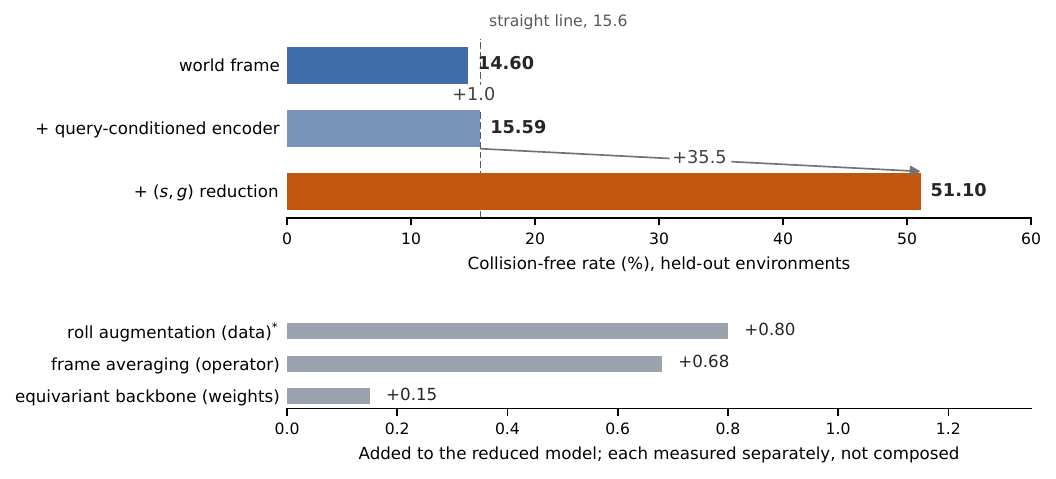}
\caption{\textbf{Sixty environments, trained to convergence.} \emph{Top}: the
frame determined by the start and the goal is worth $+35.5$ points over the
query-conditioned encoder and $+36.5$ over the plain world frame. Showing the
obstacle encoder the same query, with no change of frame, is worth $+1.0$. The
dashed line is a straight segment from start to goal on the identical problems,
which the world-frame model does not beat. \emph{Bottom}: the three mechanisms
for the one degree of freedom the frame cannot remove. Each is measured against
its own baseline. The three are alternatives rather than a decomposition, so
they are not composed here and were never composed in the experiments. Two of
them are converged measurements. Roll augmentation, starred, is the 20-epoch
grid, where its effect was already not significant across three seeds and was
not run further.}
\label{fig:teaser}
\end{figure}

\section{Related Work}

\textbf{Classical planners.} Sampling-based~\cite{kuffner2000rrtconnect} and
optimization-based~\cite{schulman2014trajopt,zucker2013chomp} planners act
directly on geometry, so a rigid transformation of the problem
transforms the solution. They offer probabilistic completeness or local
convergence, at a per-query cost that follows from searching or optimizing at
inference time. We use RRT-Connect and CHOMP to generate expert trajectories,
and all three to bound what is achievable on our benchmark.

\textbf{Learning-based planners.} MPNet~\cite{qureshi2019mpnet} and Motion Policy
Networks~\cite{fishman2022mpinets} move planning into a network evaluated in
milliseconds. Diffuser~\cite{janner2022diffuser}, Motion Planning
Diffusion~\cite{carvalho2023mpd} and Diffusion Policy~\cite{chi2023dp} generate
whole trajectories with a denoising model, and flow
matching~\cite{lipman2023fm} replaces
the denoising chain with a learned velocity field, which is the objective we use.
None of them inherits the symmetry that the classical planners get from the
geometry for free.

\textbf{Restoring the symmetry.} The general framing is
geometric, and equivariant models are known to carry a
strict generalisation benefit under the right conditions~\cite{elesedy2021strict}.
Group-equivariant
convolution~\cite{cohen2016gcnn}, tensor field networks~\cite{thomas2018tfn},
SE(3)-Transformers~\cite{fuchs2020se3t} and vector neurons~\cite{deng2021vn}
build the constraint into the weights, and robot learning has adopted them
widely~\cite{ryu2023edf,wang2024equidp}.
Frame averaging~\cite{puny2022fa} and learned
canonicalization~\cite{kaba2023canon} act at inference instead, and the
projection property our bound rests on is theirs. The Lie
derivative~\cite{gruver2023lie} measures how equivariant a trained model became.
That is the quantity we bound, and we find it is not sufficient to predict what
a mechanism is worth. Frames read off the start and the goal are common practice in manipulation
policy learning, where they are used without being measured. We do not claim
the construction. We claim what it is worth, against a floor and against the
three alternatives.

\section{Method}
In this section, we begin by defining the generative planning objective and the symmetry the task
carries (Sec.~\ref{sec:problem}). Sec.~\ref{sec:decompose} then splits $\SE$ into
the part a query fixes and the part it leaves behind, and identifies the
remainder as $\mathrm{SO}(2)$. That remainder does not yield to the same
treatment, and Sec.~\ref{sec:obstruction} gives two obstructions to it, one
topological and one that applies to scenes carrying a symmetry of their own.
Finally, Sec.~\ref{sec:threeways} sets out the three places a symmetry can be
imposed, in the training data, in the inference operator, and in the weights, and
defines the non-equivariance residual that bounds what the operator can remove.

\subsection{Problem formulation}
\label{sec:problem}

We formulate motion planning in the workspace $\mathbb{R}^3$. A task is the
tuple $\mathcal{P} = (\svec, \gvec, \mathcal{O})$, where
$\svec, \gvec \in \mathcal{C}$ are the start and goal configurations,
$\mathcal{C} = \mathbb{R}^3$ for the point mass studied throughout and
$\mathcal{C} = \SE$ for the rigid-body domain of
Sec.~\ref{sec:replications},
and $\mathcal{O}$ is a set of convex obstacles inducing a signed distance field
$d_{\mathcal{O}} : \mathbb{R}^3 \to \mathbb{R}$. The robot is a sphere of radius
$\rho$ centered at $p$. For the rigid body of Sec.~\ref{sec:replications} it is a
union of five such spheres rigidly attached to the pose, at centres $c_j(p)$; the
point mass is the single-sphere case $c_1(p) = p$. A trajectory is a waypoint
sequence $x = [p_1, \dots, p_N] \in \mathcal{C}^N$ and $\gamma(x)$ the path it
traces, piecewise linear in position and geodesic in orientation. The feasible
set of $\mathcal{P}$ is
\begin{equation}
\begin{aligned}
  \mathcal{S}(\mathcal{P}) = \big\{\, x \;:\;\; & p_1 = \svec,\; p_N = \gvec, \\
  & \min_j d_{\mathcal{O}}\big(c_j(p)\big) \ge \rho
    \;\; \forall\, p \in \gamma(x) \,\big\}.
\end{aligned}
  \label{eq:feasible}
\end{equation}
The objective is to learn a generative model
$p_\theta(x \mid \mathcal{P})$ supported on $\mathcal{S}(\mathcal{P})$.

A rigid transformation $T \in \SE$ acts on a task by
$T\mathcal{P} = (T\svec, T\gvec, T\mathcal{O})$ and on a trajectory waypointwise,
$(Tx)_i = Tp_i$. The task is then exactly $\SE$-equivariant:
\begin{equation}
  \mathcal{S}(T\mathcal{P}) = T\,\mathcal{S}(\mathcal{P}),
  \qquad \forall\, T \in \SE ,
  \label{eq:equivariance}
\end{equation}
and a generative planner should inherit this,
\begin{equation}
  p_\theta(Tx \mid T\mathcal{P}) = p_\theta(x \mid \mathcal{P}),
  \qquad \forall\, T \in \SE .
  \label{eq:planner-equivariance}
\end{equation}
We use conditional flow matching~\cite{lipman2023fm}, which learns a
time-dependent velocity field $v_\theta(x_t, t, \mathcal{P})$ whose integral
transports a simple prior to $p_\theta(x \mid \mathcal{P})$. Velocities are free
vectors, therefore the corresponding condition on the field is
\begin{equation}
  v_\theta(Tx_t, t, T\mathcal{P}) = R\, v_\theta(x_t, t, \mathcal{P}),
  \quad T = (R, \tau) \in \SE ,
  \label{eq:field-equivariance}
\end{equation}
which yields Eq.~\eqref{eq:planner-equivariance} whenever the prior transported
by the flow is itself invariant under the action. Nothing in the flow matching
objective enforces either one, and Sec.~\ref{sec:residual} measures how far a
trained field falls from Eq.~\eqref{eq:field-equivariance}.

\subsection{Decomposing the group}
\label{sec:decompose}

The motion planning task gives us a canonical frame at initialisation, and at the cost of one cross product per query. Unfortunately the canonicalization is not entirely complete, since fixing
the rotation about the first axis needs a second reference axis, and the query
does not give us one. We decompose as follows.

We declare $m = \tfrac{1}{2}(\svec + \gvec)$, $\ell = \|\gvec - \svec\| > 0$
and $\hat{u} = (\gvec - \svec)/\ell$. The second axis is not determined by the
query, and ours is read off the world basis vector $e_k$ that $\hat{u}$ is least
aligned with,
\begin{equation}
  k = \operatorname*{arg\,min}_{i} |\hat{u}_i| ,
  \qquad
  \hat{n} = \frac{\hat{u} \times e_k}{\|\hat{u} \times e_k\|} .
  \label{eq:second-axis}
\end{equation}
Choosing $k$ this way bounds $\hat{u}_k^2 \le \tfrac{1}{3}$, so
$\|\hat{u} \times e_k\|^2 = 1 - \hat{u}_k^2 \ge \tfrac{2}{3}$ and the denominator
stays away from zero for every query. Collecting the three axes into
$R_{\mathcal{P}} \in \mathrm{SO}(3)$, with
$R_{\mathcal{P}}^{\top} = [\,\hat{u}\;\; \hat{n}\;\; \hat{u} \times \hat{n}\,]$,
the reduction is the rigid map
\begin{equation}
  \Phi_{\mathcal{P}}(p) = R_{\mathcal{P}}\,(p - m),
  \label{eq:reduction}
\end{equation}
under which the endpoints are no longer inputs but constants,
\begin{equation}
  \Phi_{\mathcal{P}}(\svec) = \big({-\tfrac{\ell}{2}},\, 0,\, 0\big),
  \qquad
  \Phi_{\mathcal{P}}(\gvec) = \big({+\tfrac{\ell}{2}},\, 0,\, 0\big).
  \label{eq:endpoints}
\end{equation}

Hence the network never sees a world coordinate. The task is carried into the
frame, $\hat{x}_t = \Phi_{\mathcal{P}}(x_t)$ and
$\hat{\mathcal{O}} = \Phi_{\mathcal{P}}\mathcal{O}$, obstacle centres transforming
as points, box edges as free vectors under $R_{\mathcal{P}}$ alone, and radii not
at all. The field is evaluated entirely inside the frame,
$\hat{v} = v_\theta(\hat{x}_t, t \mid \ell, \hat{\mathcal{O}})$, where by
Eq.~\eqref{eq:endpoints} the query enters as the single scalar $\ell$ in place of
the six numbers of $(\svec, \gvec)$. The output is returned to the world by the
rotation alone, $v = R_{\mathcal{P}}^{\top}\hat{v}$, velocities being free
vectors.

Eq.~\eqref{eq:reduction} costs one cross product per query, constrains nothing
about $v_\theta$, and holds at initialisation. It fixes five of the six degrees
of freedom of $\SE$, three of translation through $m$ and two of rotation
through $\hat{u}$. It does not fix the sixth, and
Eq.~\eqref{eq:second-axis} reads the world basis, so it turns with the world.

\begin{proposition}[The residual is $\mathrm{SO}(2)$]
\label{prop:residual}
Let $R_1(\theta)$ denote the rotation by $\theta$ about the first coordinate
axis. For every task $\mathcal{P}$ with $\ell > 0$ and every $T \in \SE$ there is a
$\theta = \theta(T, \mathcal{P})$ with
\begin{equation}
  \Phi_{T\mathcal{P}} \circ T \;=\; R_1(\theta) \circ \Phi_{\mathcal{P}} .
  \label{eq:residual}
\end{equation}
\end{proposition}
\begin{proof}
$\|T\gvec - T\svec\| = \ell$, so by Eq.~\eqref{eq:endpoints} both
$\Phi_{T\mathcal{P}} \circ T$ and $\Phi_{\mathcal{P}}$ carry $\svec$ to
$(-\tfrac{\ell}{2},0,0)$ and $\gvec$ to $(+\tfrac{\ell}{2},0,0)$. Their composite
$Q = \Phi_{T\mathcal{P}} \circ T \circ \Phi_{\mathcal{P}}^{-1}$ is, therefore, an
orientation-preserving isometry fixing two distinct points, hence a rotation
about the line joining them, which is the first coordinate axis.
\end{proof}

Nothing in the proof refers to Eq.~\eqref{eq:second-axis}. The residual is
$\mathrm{SO}(2)$ for every rule selecting the second axis, and the rule decides
only which $\theta$. What Eq.~\eqref{eq:planner-equivariance} still asks of a
planner in this frame is equivariance to $\mathrm{SO}(2)$, then, and not to
nothing. One rule would escape it, a rule continuous in $\mathcal{P}$ and
returning the same $\hat{n}$ for $\mathcal{P}$ and $T\mathcal{P}$. Ours is not
continuous: it jumps where the minimising $k$ in Eq.~\eqref{eq:second-axis} ties.
Sec.~\ref{sec:obstruction} is why no rule does better.

Eq.~\eqref{eq:reduction} can be read as input normalisation rather than as a
group action, and the two readings name the same operation. The decomposition
invites it: the three translations alone are worth $+12.3$ points, and recentring
on $m$ contains no rotation. What the symmetry reading adds is not a different
computation but an account of which parts come off exactly, what is left when
they do, and why the remainder is $\mathrm{SO}(2)$ rather than a nuisance to be
tuned away.

\subsection{Why no continuous rule removes the remainder}
\label{sec:obstruction}
Sec.~\ref{sec:decompose} removed five degrees of freedom by reading a frame off
the query. Whether the sixth yields to the same treatment, under some better
choice of second axis, is not a question about our construction. Two obstructions
answer it, one topological and one that applies to an individual scene. Neither
is original to us. We state them because the rest of the paper is an empirical
answer to what they leave open.

\begin{proposition}[Topological obstruction]
\label{prop:two}
There is no continuous $\hat{n} : S^2 \to S^2$ with $\hat{n}(\hat{u}) \perp
\hat{u}$ for every $\hat{u} \in S^2$. Every rule selecting the second axis,
Eq.~\eqref{eq:second-axis} among them, is therefore discontinuous somewhere on
the sphere of query directions.
\end{proposition}
\begin{proof}
Such an $\hat{n}$ is a continuous nowhere-vanishing tangent vector field on
$S^2$, which the hairy ball theorem forbids~\cite{milnor1978hairy}.
\end{proof}

\begin{proposition}[Scenes with their own symmetry]
\label{prop:three}
Let $\hat{\mathcal{O}}$ be a reduced scene with
$R_1(\pi)\,\hat{\mathcal{O}} = \hat{\mathcal{O}}$. Then no rule reading only the
reduced task can distinguish the frame $R_{\mathcal{P}}$ from
$R_1(\pi)\,R_{\mathcal{P}}$.
\end{proposition}
\begin{proof}
$R_1(\pi)$ fixes both endpoints of Eq.~\eqref{eq:endpoints} and fixes
$\hat{\mathcal{O}}$ by hypothesis, so the two frames present the rule with
identical arguments.
\end{proof}

Neither proposition rules out a discontinuous rule that reads the
scene~\cite{kaba2023canon}, and outside the measure-zero set of
Prop.~\ref{prop:three} such a rule can fix the roll. What we show is that the
remainder cannot be removed the way the other five degrees of freedom were, in
closed form from the query and continuously. Sec.~\ref{sec:results} measures the
three mechanisms that remain, and none of them recovers a point.

\subsection{Three ways to impose the symmetry}
\label{sec:threeways}
A symmetry that the problem has and the model lacks can be imposed in three
places. In the data, by training on randomly transformed copies of each problem,
which teaches the symmetry without guaranteeing it. In the inference operator, by
averaging the learned field over the group, which guarantees it for the averaged
field without changing what was learned. In the weights, by restricting the
hypothesis class to fields equivariant by construction, which guarantees it at
initialisation and constrains optimisation throughout.

Averaging a field $f$ over the $K$ rolls $\theta_k = 2\pi k / K$ gives
\begin{equation}
  (\mathcal{A}f)(\hat{x}) = \frac{1}{K} \sum_{k=0}^{K-1}
  R_1(\theta_k)^{\top} f\big(R_1(\theta_k)\,\hat{x}\big) ,
  \label{eq:frame-averaging}
\end{equation}
the orthogonal projection of $f$ onto the fields equivariant to the cyclic
subgroup $C_K \subset \mathrm{SO}(2)$~\cite{puny2022fa}.

The first two mechanisms act on a field that has already been trained, and how
far such a field sits from equivariant is measurable rather than assumed. We
write
\begin{equation}
  r \;=\; \frac{\|f - \mathcal{A}f\|}{\|\mathcal{A}f\|}
  \label{eq:residual-def}
\end{equation}
for the non-equivariance residual of $f$, which costs one forward pass per
rotation and no retraining. It bounds what the second mechanism can do.

\begin{proposition}[Symmetrisation budget]
\label{prop:budget}
Let $f^\ast$ be equivariant and $\mathcal{A}$ the projection of
Eq.~\eqref{eq:frame-averaging}. Then
\begin{equation}
  \|f - f^\ast\|^2 = \|\mathcal{A}f - f^\ast\|^2 + r^2\|\mathcal{A}f\|^2 .
  \label{eq:budget}
\end{equation}
The squared error that any equivariantisation of $f$ can remove is at most
$r^2\|\mathcal{A}f\|^2$, and $\mathcal{A}$ attains it. The remainder
$\|\mathcal{A}f - f^\ast\|$ is equivariant error, which no symmetrisation of $f$
reduces. Every map carrying $f$ to an equivariant field moves it by at least
$r\|f\|/\sqrt{1+r^2}$, with equality only for $\mathcal{A}$.
\end{proposition}
\begin{proof}
$\mathcal{A}$ is an orthogonal projection, so $f - \mathcal{A}f$ is orthogonal to
the equivariant subspace, which contains $\mathcal{A}f - f^\ast$.
\end{proof}

\section{Setup and Calibration}
\subsection{Benchmark and protocol}

PointMass3D places forty obstacles, twenty spheres and twenty oriented boxes, in
$[-1,1]^3$. The robot is a sphere of radius $\rho = 0.03$. Collision checking uses an
analytic signed distance field, and $\gamma(x)$ is resampled at a spacing of
$0.01$ before it is tested rather than evaluated at the $64$ waypoints alone.
Since $d_{\mathcal{O}}$ is $1$-Lipschitz, every point of $\gamma(x)$ lies within
$0.005$ of a tested sample, so the reported rates differ from the continuous ones
by at most a $0.005$ perturbation of $\rho$, a sixth of the robot radius. A trajectory
is $N=64$ waypoints in $\mathbb{R}^3$.

Expert paths come from RRT-Connect, then random shortcutting, uniform
resampling, and CHOMP refinement. The dataset is $300$ environments, each with
$600$ start-goal pairs and $30$ paths per pair including time reversals.
Environments $250$ to $299$ are held out and never trained on. Unless stated
otherwise, models train on the first $60$ of the remaining $250$; the scaling
study in Sec.~\ref{sec:scaling} varies that count.

The planner is a conditional flow-matching model with $2.16$M parameters: a
dilated temporal convolutional trunk~\cite{oord2016wavenet} with FiLM
conditioning~\cite{perez2018film}
and a PointNet-style obstacle encoder~\cite{qi2017pointnet}. The encoder
max-pools all forty obstacles into a single $128$-dimensional vector, and FiLM
applies that vector identically to every one of the $64$ waypoints.

Every model is evaluated on the same $500$ problems: $50$ unseen environments,
$10$ distinct start-goal pairs each, $20$ samples per pair, eight Euler steps. We
report two metrics and never conflate them. A sample is \emph{free} if the whole
path is collision-free. \emph{Best-of-20} asks whether any of the twenty samples
for a query is free. They differ by about thirty points. Uncertainty is quoted
across seeds for claims about a method. For comparisons against a fixed baseline
measured on the same problems we use a paired bootstrap over environments.

\begin{table}[t]
\centering
\footnotesize
\setlength{\tabcolsep}{3.5pt}
\begin{tabular}{@{}lrrr@{}}
\toprule
Planner & Free (\%) & Clearance & s/query \\
\midrule
straight line                   & 15.6 & 0.055 & $<$0.001 \\
TrajOpt                         & 74.0 & 0.026 & 0.352 \\
CHOMP                           & 88.8 & 0.019 & 0.394 \\
RRT-Connect                     & 99.6 & 0.017 & 0.284 \\
expert (RRT+CHOMP)              & 100.0 & 0.037 & 0.342 \\
\midrule
flow, world frame               & 14.60 & $-0.071$ & 0.053 \\
flow, $(\svec,\gvec)$ reduction & \textbf{51.10} & $-0.002$ & 0.063 \\
\bottomrule
\end{tabular}
\caption{The same $500$ unseen problems solved five other ways. The learned rows
are converged, three-seed means at $60$ environments. Classical timings are one
CPU core and one query; the learned rows are one GPU and twenty samples, so the
time column is not like-for-like. Clearance is $d_{\mathcal{O}} - \rho$, the
margin by which the robot misses the nearest obstacle, minimised along the
densely resampled path; it is negative exactly when the path collides. Classical
rows average it over solved problems only, learned rows over all samples, which
is why the straight line reads $+0.055$ while failing $84.4\%$ of them. The
expert row is the filtered training distribution rather than a planner that never
fails.}
\label{tab:floors}
\end{table}

\subsection{The floor, and a paired test against it}
A learned planner is worth what it adds to the cheapest thing that solves the
same problems. Joining start to goal with a straight segment is collision-free on
$15.6\%$ of the $500$ held-out problems, so the clutter is dense and not so dense
that the direct path is usually blocked. The world-frame model converges to
$14.60 \pm 0.20\%$.

Quoting the across-seed spread here would answer a different question. That
spread describes variation between training runs, not the uncertainty of a
comparison against a fixed baseline measured on the identical problems. Both
quantities track per-environment clutter, so we resample environments and
recompute the difference within each replicate. The paired bootstrap gives
$-1.02$ points, a $95\%$ interval of $[-3.06, +0.90]$, and a standard error of
$1.01$, less than half the $2.2$ obtained by resampling either rate alone, and
still spanning zero. What we defend is therefore that training on world
coordinates does not beat joining the endpoints with a line, and not that it
loses to one. At $20$ epochs even that was unavailable ($15.37 \pm 0.36$).
Optimisation is what makes the comparison resolvable, not what creates the
deficit.

None of this is what the planner is worth to a practitioner. Behind a collision
check, best-of-$20$, the reduced model solves $87.2\%$ of held-out queries at $60$
environments against $15.8\%$ for the best untrained prior. That is the number to
read if the question is whether the sampler is useful rather than what the
representation buys, subject to Table~\ref{tab:floors}.

\section{Results}
\label{sec:results}
\subsection{Scaling}
\label{sec:scaling}
Across a $12.5\times$ increase in training environments the world-frame model
goes from $12.9$ to $15.4$. The reduction goes from $19.2$ to $45.5$, and has
not plateaued. On this grid the frame is worth $+30.1$ points, splitting into
$+12.3$ for the three translations and $+18.0$ for the two rotations. Seed-matched, the difference grows
monotonically at every size: $+6.37 \pm 0.17$, $+20.87 \pm 1.54$,
$+25.85 \pm 0.51$ and $+30.08 \pm 0.44$, quoting the standard error of the
per-seed differences rather than combining the two arms' spreads in quadrature.

The world-frame model is not short of data. Training it on random $\SE$ motions
of the whole problem, which is the data mechanism of Sec.~\ref{sec:threeways},
moves it by at most $2.1$ points and leaves it at $17.5$. These cells hold the
budget at $20$ epochs so that representation is the only thing varying, which is
why $45.5$ at $250$ environments sits below the $51.10$ that
Sec.~\ref{sec:localgeom} reports at $60$. Trained out, the $250$-environment
model reaches $55.8$.

\begin{figure}[t]
\centering
\includegraphics[width=\columnwidth]{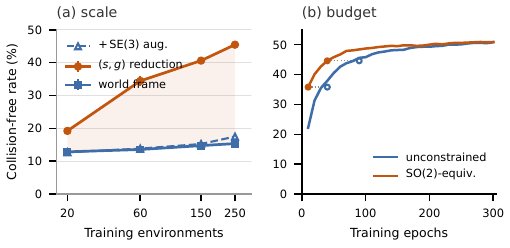}
\caption{\textbf{(a) Held-out collision-free rate against training-set size,
at a fixed budget of $20$ epochs.} The world-frame model moves $2.5$ points
across a $12.5\times$ increase in environments, from $12.9$ to $15.4$, while the
reduction runs from $19.2$ to $45.5$ without plateauing. Holding the budget fixed
is what makes representation the only variable, and it is also why these are not
convergence levels: trained out, the $250$-environment reduced model reaches
$55.8$. Every world-frame and reduction cell is a mean over three seeds.
\textbf{(b) The same rate against training epochs, $60$ environments.} The
SO(2)-equivariant backbone reaches any given level two to three times sooner than the
unconstrained model in the same frame, and the two converge to within $0.15$
points. Seed-matched over three seeds.}
\label{fig:scaling}
\end{figure}

\subsection{Control: a query-conditioned encoder}
\label{sec:querycond}
The gap admits a deflationary reading. In the world frame the obstacle encoder
sees obstacles and nothing else, so its output is constant across every query in
an environment, at a measured within-environment standard deviation of
$0.000000$. That is a property of the architecture rather than of training. Under
the reduction the same encoder sees those obstacles expressed in the query's own
frame, so its code varies with the query. On this reading the $+36.5$ measures
letting the encoder see the query, and the change of frame is incidental.

We built the model that separates the two. It stays in the world frame
throughout, and the raw query is concatenated to every obstacle before the pool,
so the scene code depends on the query with no change of frame anywhere. It
carries $0.07\%$ more parameters.

It converges to $15.59 \pm 0.38\%$ against the world frame's
$14.60 \pm 0.20$, seed-matched $+1.01$ and positive on all three seeds. Query
dependence is worth about one point of the $36.5$. The level it reaches is the
more useful number. At $15.59$ it lands on the straight line's $15.6$, so showing
the encoder the query is enough to reach the floor and not enough to pass it.

\subsection{Control: local geometry, and how it scales}
\label{sec:localgeom}
One scene code applied identically to all $64$ waypoints leaves a waypoint no
channel through which to ask what is near \emph{it}, and collision avoidance is
that question. Sec.~\ref{sec:querycond} shows the reduction addresses this
indirectly. It can also be addressed directly. We append to each waypoint the
signed distance to the scene and its gradient there, which for the SO(2)-equivariant backbone splits by irreducible representation: $d_{\mathcal{O}}$ and the axial
component of $\nabla d_{\mathcal{O}}$ are invariant, the normal component
rotates. This is a featurisation and not an oracle, since the signed distance is
a deterministic function of primitives the network already receives. What it
assumes is that those primitives are available at inference, which holds on this
benchmark and fails for a planner operating on perception.

Table~\ref{tab:twobytwo} crosses the two. Local geometry is worth $+40.0$ points
in the world frame, more than the frame itself is worth, so the representation
change studied here is not the largest effect available on this benchmark. The
two are not substitutes either: with local geometry supplied, the frame still
adds $+24.8$. They are sub-additive. Treating them as independent predicts
$14.60 + 36.5 + 40.0 = 91.1$ and the measured value is $79.37$, because both
address the same difficulty by different routes.

What the frame is worth beside local geometry decays with data. Differencing
within seed at a matched $80$-epoch budget, it is $+22.37 \pm 0.46$ points at
$60$ environments and $+17.24 \pm 0.17$ at $250$, two seeds at each scale. The
$5.1$-point decay is an order of magnitude larger than either standard error.
Local geometry is what scales over that range, gaining $11$ points while the
frame beside it loses $5$.

\begin{table}[t]
\centering
\footnotesize
\setlength{\tabcolsep}{3.5pt}
\begin{tabular}{@{}lrrr@{}}
\toprule
$60$ envs, converged & global only & $+$ local & geometry adds \\
\midrule
world frame                    & 14.60 & 54.54 & $+40.0$ \\
$(\svec,\gvec)$ reduction      & 51.10 & 79.37 & $+28.3$ \\
\quad $+$ equivariant backbone & 51.27 & \textbf{82.13} & $+30.9$ \\
\midrule
\emph{the frame contributes}   & $+36.5$ & $+24.8$ & \\
\bottomrule
\end{tabular}
\caption{Local geometry against the query frame. Every cell is a mean over two or
three seeds. The two local-geometry columns are the least seed-sensitive
measurements in this work, with across-seed standard errors of $0.01$ and $0.10$
against $0.20$ and $0.25$ without. The equivariant row is a seed-matched pair,
$51.27$ against $51.12$ on the same three seeds, so what that backbone is worth
is the $+0.15$ between them and not $51.27$ minus the three-seed mean $51.10$
above it.}
\label{tab:twobytwo}
\end{table}

\subsection{Replications}
\label{sec:replications}
Three things vary, one at a time. A second state space: an L-shaped rigid body in
the same clutter, whose state is a full pose rather than a point, so the
points-versus-vectors distinction moves inside the trajectory. Its trivial
baseline is geodesic interpolation, straight in position and slerp in rotation.
The world-frame model never clears that floor at either budget; the reduction
clears it and is still climbing. A second generative objective: replacing flow
matching with DDPM at matched network evaluations leaves the gap at $+37.01$
against $+36.50$, so the effect does not belong to the sampler. A range of
clutter: on freshly generated layouts at densities neither model was trained on,
the gap runs from $+29.4$ at eight obstacles to $+41.2$ at twenty-four.

The rigid-body cells are the weakest measurements in this paper. One seed, and
both models below their own floor at $20$ epochs. What replicates there is the
ratio and not the level: at matched budgets the reduced model is $2.9\times$ the
world-frame model in both domains.

\begin{table}[t]
\centering
\footnotesize
\setlength{\tabcolsep}{3.5pt}
\begin{tabular}{@{}llrrr@{}}
\toprule
Replication & Condition & World & Reduced & Gap \\
\midrule
\multicolumn{5}{l}{\textbf{Second state space.} $\SE$ rigid body, floor $4.8$}\\
            & $20$ ep.       & 2.8 & 6.2 & $+3.4$ \\
            & $155$/$159$ ep.& 3.1 & \textbf{10.0} & $+6.9$ \\
\midrule
\multicolumn{5}{l}{\textbf{Other objective.} DDPM, matched network evaluations}\\
            & $20$ ep.       & 12.27 & 28.97 & $+16.71$ \\
            & $300$ ep.      & 13.51 & 50.52 & $\mathbf{+37.01}$ \\
\midrule
\bottomrule
\end{tabular}
\caption{The effect survives a change of state space, and a change of generative
objective. The rigid body is an L-shaped union
of five spheres whose state is a full pose, chosen to have trivial symmetry so
the domain is not secretly easier; both models are below their own floor at $20$
epochs and only the reduced one clears it. The second column is each model's
own best-validation checkpoint, epoch $155$ for the world frame and $159$ for the
reduction, so neither had stopped improving when the runs ended. Single seed except the DDPM $20$-epoch row.}
\label{tab:replications}
\end{table}

\subsection{The residual, and what it bounds}
\label{sec:residual}
Proposition~\ref{prop:budget} bounds what symmetrising a trained field can remove. The reduced
model has $r = 0.017$. That bound is under $3 \times 10^{-4}$ of the field's
squared magnitude, and the model fails on $48.9\%$ of its samples. The two
quantities are not the same order, so no symmetrisation of this field closes
that gap. The measurement costs one forward pass per rotation and no
retraining.

The residual does not forecast what a mechanism is worth. Measured under the
full $\mathrm{SE}(3)$ action, which makes it defined for the world-frame model
as well, augmentation drives $r$ from $0.0881$ to $0.0181$ and buys $2.1$
points. The reduction, at $0.0174$, buys $30.1$. Both residuals are single
measurements on one checkpoint each, since $r$ is a property of a particular
trained field, and they differ by less than the gap between the $K=9$ and
$K=32$ estimates of either.

\section{Limitations}

\begin{itemize}\itemsep3pt
  \item \textbf{The group must act on the state space, not just the workspace.}
  This holds for a point mass and for the rigid body of Table~\ref{tab:replications}.
  It fails for a manipulator planning in joint space: the query still defines a
  frame in the workspace, but a rigid motion of the world induces no clean action
  on joint angles. Mobile bases, free-flying bodies and end-effector-space
  planning satisfy the condition. Joint-space planning for arms does not.

  \item \textbf{The benchmark was chosen for isolability, not for difficulty.}
  RRT-Connect solves $99.6\%$ of these problems in under $0.3$\,s on one CPU
  core, and on this benchmark one should use RRT-Connect. A point mass in a box
  is where clutter, budget, seed and representation can be varied one at a time,
  which is the only reason any number here is attributable to the representation.
  The sampler is not a toy at that: behind a collision check it solves $87.2\%$
  of held-out queries.

  \item \textbf{One workspace family, one clutter generator, one architecture.}
  The claim that the sixth degree of freedom needs no treatment relies on the
  training distribution supplying implicit roll diversity through many start-goal
  directions. A dataset with few directions per environment could reverse it.

  \item \textbf{One cell of the design was not run.} The SO(2)-equivariant backbone
  was only ever trained inside the reduced frame, so what we establish is that
  architectural equivariance adds nothing once the frame is present. That is not
  the same as showing it cannot supply what the frame supplies. A world-frame
  equivariant model would separate the two readings, and it is the experiment we
  would run next.

  \item \textbf{Seeds are not uniform.} The $60$-environment cells carry three
  seeds, the local-geometry cells two or three, and the $250$-environment grid
  three. Only the converged $250$-environment run and the rigid-body cells of
  Table~\ref{tab:replications} carry one. Every claim we lean on has at least two.

  \item \textbf{The frame's value beside local geometry falls with data.} It is
  $+22.4$ points at $60$ environments and $+17.2$ at $250$ on a matched budget.
  We have not found where that decay stops, and a reader extrapolating to much
  larger datasets should expect the margin to narrow rather than hold.

  \item \textbf{Local geometry assumes explicit primitives at inference.} The
  signed distance is a deterministic function of obstacle data the network
  already receives, so it is a featurisation rather than an oracle, but it is
  unavailable to a planner working from perception.
\end{itemize}

\section{Conclusion}
Five of the six degrees of freedom of $\mathrm{SE}(3)$ were removable exactly,
at initialisation and at the cost of one cross product per query, by a frame the planning query already
supplied. That was worth $+36.5$ points of held-out collision-free rate,
against a world-frame model that did not beat a straight line. The sixth was
not removable by any continuous rule. Frame averaging and the SO(2)-equivariant backbone were trained to convergence, and neither recovered a point. Roll
augmentation was measured only at a fixed budget, where its $+0.8$ already sat
inside the spread across three seeds. The SO(2)-equivariant backbone still earned
its place, reaching any given level two to three times sooner, so the symmetry
bought optimisation and not accuracy. Two models with indistinguishable
residuals differed by $28$ points, so the residual was not sufficient to
predict what a mechanism was worth. Giving each waypoint its distance and
direction to the nearest obstacle was worth more than the frame was, and with
that information supplied the frame still added $+24.8$ points.

The construction transfers wherever the group acts on the state space and the
query names two points, which covers mobile bases, free-flying bodies and
end-effector-space planning, and excludes joint-space planning for arms. It
costs one cross product per query and constrains nothing about the network. The
payoff should be largest where relearning the same motion at every pose is
most expensive, which means a model small relative to its workspace and a
narrow conditioning path from the scene to the trajectory. Both were true here,
and the $+36.5$ points shrank to $+24.8$ as soon as we widened that path with
local geometry.

\section*{Acknowledgements}

The author thanks the IWIN-FINS Laboratory at Shanghai Jiao Tong
University for the computational resources used in this work.

The software implementation was carried out with Claude Code under the
author's direction. The classical planner implementations, used to
generate expert trajectories and to bound what is achievable, follow the
published
algorithms~\cite{kuffner2000rrtconnect,schulman2014trajopt,zucker2013chomp}
and their public reference implementations. The problem formulation, the
theoretical analysis and the experimental design are the author's own.
Claude was used for language and grammar editing of the manuscript.
Every reported number is produced by the released code, whose geometric
and equivariance-critical components are verified by property-based tests
that run on untrained networks and include negative controls.


\small
\begin{thebibliography}{99}

\bibitem{carvalho2023mpd}
J.~Carvalho, A.~T. Le, M.~Baierl, D.~Koert, and J.~Peters.
\newblock Motion planning diffusion: Learning and planning of robot motions with
  diffusion models.
\newblock In {\em IROS}, 2023.

\bibitem{chi2023dp}
C.~Chi, S.~Feng, Y.~Du, Z.~Xu, E.~Cousineau, B.~Burchfiel, and S.~Song.
\newblock Diffusion policy: Visuomotor policy learning via action diffusion.
\newblock In {\em RSS}, 2023.

\bibitem{cohen2016gcnn}
T.~S. Cohen and M.~Welling.
\newblock Group equivariant convolutional networks.
\newblock In {\em ICML}, 2016.

\bibitem{deng2021vn}
C.~Deng, O.~Litany, Y.~Duan, A.~Poulenard, A.~Tagliasacchi, and L.~Guibas.
\newblock Vector neurons: A general framework for {SO(3)}-equivariant networks.
\newblock In {\em ICCV}, 2021.

\bibitem{elesedy2021strict}
B.~Elesedy and S.~Zaidi.
\newblock Provably strict generalisation benefit for equivariant models.
\newblock In {\em ICML}, 2021.

\bibitem{fishman2022mpinets}
A.~Fishman, A.~Murali, C.~Eppner, B.~Peele, B.~Boots, and D.~Fox.
\newblock Motion policy networks.
\newblock In {\em CoRL}, 2022.

\bibitem{fuchs2020se3t}
F.~B. Fuchs, D.~E. Worrall, V.~Fischer, and M.~Welling.
\newblock {SE(3)}-transformers: 3{D} roto-translation equivariant attention
  networks.
\newblock In {\em NeurIPS}, 2020.

\bibitem{gruver2023lie}
N.~Gruver, M.~Finzi, M.~Goldblum, and A.~G. Wilson.
\newblock The {L}ie derivative for measuring learned equivariance.
\newblock In {\em ICLR}, 2023.

\bibitem{janner2022diffuser}
M.~Janner, Y.~Du, J.~B. Tenenbaum, and S.~Levine.
\newblock Planning with diffusion for flexible behavior synthesis.
\newblock In {\em ICML}, 2022.

\bibitem{kaba2023canon}
S.-O. Kaba, A.~K. Mondal, Y.~Zhang, Y.~Bengio, and S.~Ravanbakhsh.
\newblock Equivariance with learned canonicalization functions.
\newblock In {\em ICML}, 2023.

\bibitem{kuffner2000rrtconnect}
J.~J. Kuffner and S.~M. LaValle.
\newblock {RRT}-connect: An efficient approach to single-query path planning.
\newblock In {\em ICRA}, 2000.

\bibitem{lipman2023fm}
Y.~Lipman, R.~T.~Q. Chen, H.~Ben-Hamu, M.~Nickel, and M.~Le.
\newblock Flow matching for generative modeling.
\newblock In {\em ICLR}, 2023.

\bibitem{milnor1978hairy}
J.~Milnor.
\newblock Analytic proofs of the ``hairy ball theorem'' and the {B}rouwer fixed
  point theorem.
\newblock {\em Amer. Math. Monthly}, 1978.

\bibitem{oord2016wavenet}
A.~van~den Oord, S.~Dieleman, H.~Zen, K.~Simonyan, O.~Vinyals, A.~Graves,
  N.~Kalchbrenner, A.~Senior, and K.~Kavukcuoglu.
\newblock {W}ave{N}et: A generative model for raw audio.
\newblock {\em arXiv:1609.03499}, 2016.

\bibitem{perez2018film}
E.~Perez, F.~Strub, H.~de~Vries, V.~Dumoulin, and A.~Courville.
\newblock {FiLM}: Visual reasoning with a general conditioning layer.
\newblock In {\em AAAI}, 2018.

\bibitem{puny2022fa}
O.~Puny, M.~Atzmon, H.~Ben-Hamu, I.~Misra, A.~Grover, E.~J. Smith, and
  Y.~Lipman.
\newblock Frame averaging for invariant and equivariant network design.
\newblock In {\em ICLR}, 2022.

\bibitem{qi2017pointnet}
C.~R. Qi, H.~Su, K.~Mo, and L.~J. Guibas.
\newblock {P}oint{N}et: Deep learning on point sets for 3{D} classification and
  segmentation.
\newblock In {\em CVPR}, 2017.

\bibitem{qureshi2019mpnet}
A.~H. Qureshi, A.~Simeonov, M.~J. Bency, and M.~C. Yip.
\newblock Motion planning networks.
\newblock In {\em ICRA}, 2019.

\bibitem{ryu2023edf}
H.~Ryu, H.-i. Lee, J.-H. Lee, and J.~Choi.
\newblock Equivariant descriptor fields: {SE(3)}-equivariant energy-based models
  for end-to-end visual robotic manipulation learning.
\newblock In {\em ICLR}, 2023.

\bibitem{schulman2014trajopt}
J.~Schulman, Y.~Duan, J.~Ho, A.~Lee, I.~Awwal, H.~Bradlow, J.~Pan, S.~Patil,
  K.~Goldberg, and P.~Abbeel.
\newblock Motion planning with sequential convex optimization and convex
  collision checking.
\newblock {\em IJRR}, 2014.

\bibitem{thomas2018tfn}
N.~Thomas, T.~Smidt, S.~Kearnes, L.~Yang, L.~Li, K.~Kohlhoff, and P.~Riley.
\newblock Tensor field networks: Rotation- and translation-equivariant neural
  networks for 3{D} point clouds.
\newblock {\em arXiv:1802.08219}, 2018.

\bibitem{wang2024equidp}
D.~Wang, S.~Hart, D.~Surovik, T.~Kelestemur, H.~Huang, H.~Zhao, M.~Yeatman,
  J.~Wang, R.~Walters, and R.~Platt.
\newblock Equivariant diffusion policy.
\newblock In {\em CoRL}, 2024.

\bibitem{zucker2013chomp}
M.~Zucker, N.~Ratliff, A.~D. Dragan, M.~Pivtoraiko, M.~Klingensmith, C.~M.
  Dellin, J.~A. Bagnell, and S.~S. Srinivasa.
\newblock {CHOMP}: Covariant Hamiltonian optimization for motion planning.
\newblock {\em IJRR}, 2013.

\end{thebibliography}
\end{document}